\documentclass[conference]{IEEEtran}
\IEEEoverridecommandlockouts

\input{package}

\def\BibTeX{{\rm B\kern-.05em{\sc i\kern-.025em b}\kern-.08em
    T\kern-.1667em\lower.7ex\hbox{E}\kern-.125emX}}
\begin{document}

\title{An Improved Variational Method for \\
Image Denoising}

\author{
Jing-En Huang\textsuperscript{1,*}, \quad
Jia-Wei Liao\textsuperscript{2,*}, \quad
Ku-Te Lin\textsuperscript{3,*}, \quad
Yu-Ju Tsai\textsuperscript{1,*}, \quad
Mei-Heng Yueh\textsuperscript{4,$\dagger$} \\
\textsuperscript{1}\textit{National Yang Ming Chiao Tung University},
\textsuperscript{2}\textit{National Taiwan University}, \\
\textsuperscript{3}\textit{National Central University},
\textsuperscript{4}\textit{National Taiwan Normal University} \\
\small{* Equal contribution, $\dagger$ Corresponding author}
}

\maketitle

\begin{abstract}
The total variation (TV) method is an image denoising technique that aims to reduce noise by minimizing the total variation of the image, which measures the variation in pixel intensities. The TV method has been widely applied in image processing and computer vision for its ability to preserve edges and enhance image quality. In this paper, we propose a Mixed-norm TV (MixTV) model for image denoising and the associated numerical algorithm to carry out the procedure, which is particularly effective in removing several types of noise and their combinations. Our MixTV admits a unique solution and the associated numerical algorithm guarantees convergence. Numerical experiments are demonstrated to show improved effectiveness and denoising quality compared to other TV models. Such encouraging results further enhance the utility of the TV method in image processing. Our project page is available at \url{https://jing-en-huang.github.io/MixTV/}.
\end{abstract}


\section{Introduction}
Image denoising, a crucial preliminary process in image analysis, is designed with the objective of mitigating the additive image noise. Image noise, characterized by random variations of incorrect data in images, can be attributed to sources such as natural phenomena, imperfect instruments, transmission errors, and radiation factors. Image noise can be classified into two main types: dense noise and sparse noise. In dense noise, a significant proportion of the pixels in the image are affected, and the amplitude of the noise is considerable. Common types of dense noise include Gaussian noise and Poisson noise \cite{FaR13}. In sparse noise, a limited number of pixels in the image are affected, and the noise amplitude is minimal. A common type of sparse noise is salt and pepper (S\&P) noise. Moreover, there exist other variants of noise, such as speckle noise or uniform noise among others. However, real-world noise in images can be more complex and may not strictly adhere to these singular types of noise. For instance, satellite images from space telescopes can contain a mix of Gaussian and S\&P noises, while fluorescence microscope images may be intermingled with Gaussian and Poisson noises.

The process of image denoising involves eliminating noise from an image while preserving its underlying data to produce a clean image that closely resembles the original ground truth. 
A bitmap image of size $m$-by-$n$ is represented as a matrix of pixels $F_0\in\mathbb{R}^{m\times n}$, where each element corresponds to a specific position and intensity. Then, a noisy image $F$ can be modeled as $F = F_0 + \varepsilon$, where $F_0$ is the ground-truth image and $\varepsilon$ represents the noise in the image. The denoising problem aims to recover the ground-truth image $F_0$ from the noisy image $F$.

Broadly speaking, methodologies for image denoising can be categorized into four principal groups: spatial domain filtering techniques, which encompass mean filtering~\cite{ErTH19}, Wiener filtering~\cite{BeJJY10}, median filtering~\cite{PiVA13}, and bilateral filtering~\cite{ToM98}; transform domain filtering approaches, including independent component analysis~\cite{JuA01}, principal component analysis~\cite{ZhLWDG10}, wavelet transform~\cite{MaS89}, and block-matching and 3D filtering~\cite{DaKAVK07}; variational techniques, including total variation (TV) regularization~\cite{RuOF92,RuOs94,MuZJL20,LeCA07}, non-local regularization~\cite{SuCCJ14}, sparse representation~\cite{DoWLG12}, and low-rank minimization~\cite{MaI12}; and learning-based methods~\cite{ZhW17}. Although learning-based approaches have become the mainstream paradigm for image denoising, they typically require large amounts of training data to achieve strong generalization. In many scenarios, such as satellite, astronomical, medical, and historical imaging, large-scale datasets are difficult to collect. Thus, traditional mathematical image processing methods remain valuable in data-limited settings due to their explicit modeling, interpretability, and robustness.

TV-based image denoising was first introduced by Rudin, Osher, and Fatemi \cite{RuOF92}. They presented the concept of using TV as a regularization term in an optimization framework for effectively reducing noise while preserving essential image structures. Since then, the TV model has become a fundamental and widely adopted approach in image denoising, stimulating countless advances and changes in the field.

In this paper, we briefly review classic TV models for image denoising and propose a MixTV model with an associated efficient algorithm for computing denoised images. The advantages of our new model are three-fold:
\begin{enumerate}
\item Our new model admits a unique local minimizer, which is also the global minimizer and is regarded as the denoised image.
\item The convergence of our algorithm for the new model is theoretically guaranteed.
\item Comparisons of our new model with other classic TV models are demonstrated to show better effectiveness of the new model.
\end{enumerate}

\section{Background} \label{sec:2}

\subsection{TV Models for Image Denoising}
TV denoising is an efficient approach for reducing noise in various applications, particularly image denoising. This technique can be applied to effectively preserve the sharpness and clarity of edges, to enhance image quality. 
The TV models formulate the denoising problem as optimization problems consisting of two types of terms, regularization terms, and data fidelity terms. The regularization terms encourage smoothness in the image by penalizing large variations in pixel intensities, and data fidelity terms measure the fidelity of the denoised image to the noisy image, aiming to minimize the discrepancy between them. 
Let $F\in\mathbb{R}^{m\times n}$ be the noisy image, and $U\in\mathbb{R}^{m\times n}$ be the unknown ground-truth image. Denote $\f=\text{vec}(F)$ and $\bu=\text{vec}(U)$. Some classic TV models are as follows:
\begin{itemize}
    \item 1-norm TV model \cite{AlS97,ChE05,ChEPY05,ChCCNP10}:
    \begin{equation}
        \bu^* = \argmin_{\bu} \Big({\lVert D_x \bu \rVert}_1 + {\lVert D_y \bu \rVert}_1 + \mu{\lVert \bu-\f \rVert}_1 \Big), \label{eq:TVmodel1}
    \end{equation}
    \item Isotropic TV model \cite{GoO09,LoZOX15}:
    \begin{equation}
        \bu^* = \argmin_{\bu} \Big(\sum_{i} \sqrt{(D_x \bu)_i^2+(D_y \bu)_i^2}+\frac{\mu}{2}{\lVert \bu-\f \rVert}_2^2\Big), \label{eq:TVmodel2}
    \end{equation}
    \item Anisotropic TV model~\cite{ChEPY05,GoO09,LoZOX15,EsO04}:
    \begin{equation}
        \bu^* = \argmin_{\bu} \Big({\lVert D_x \bu \rVert}_1 + {\lVert D_y \bu \rVert}_1+\frac{\mu}{2}{\lVert \bu-\f \rVert}_2^2\Big), \label{eq:TVmodel3}
    \end{equation}
\end{itemize}
where $\mu$ is a positive tuning parameter, and the difference matrices $D_x$ and $D_y$ are defined as 
\begin{subequations} \label{eq:D}
\begin{equation}
D_x=G_n\otimes I_m ~ \text{and} ~ D_y=I_n\otimes G_m
\end{equation}
with
\begin{equation}
G_n = \begin{bmatrix}
-1 & 1 \\
   & \ddots & \ddots \\
   &  & -1 & 1 \\
0 & \cdots & 0 & 0
\end{bmatrix} \in\mathbb{R}^{n\times n}.
\end{equation}
\end{subequations}
The parameter $\mu$ controls the regularization strength, a smaller value of $\mu$ leads to a more regular solution, and a larger value of $\mu$ leads to better similarity to the noisy image. 

The 1-norm TV model \eqref{eq:TVmodel1} is more suitable for removing sparse noise, and 2-norm TV models \eqref{eq:TVmodel2} and \eqref{eq:TVmodel3} are more suitable for removing dense noise. 
However, the noise, in general, might be mixed with both sparse and dense ones, e.g., S\&P and Gaussian noises, which cannot be effectively removed by the above models. To remedy this issue, we are motivated to modify the data fidelity term of the denoising model \eqref{eq:TVmodel1}, which is introduced in the next section in detail.

\section{MixTV Model}
\label{sec:3}

To deal with mixed noises in images, we propose a new denoising TV model -- MixTV:
\begin{equation} \label{eq:1&2-norm TVmodel}
\bu^* = \argmin_\bu \Big({\lVert D_x \bu \rVert}_1 + {\lVert D_y \bu \rVert}_1 + \mu{\lVert \bu-\f \rVert}_1 + \alpha{\lVert \bu-\f \rVert}^2_2 \Big),
\end{equation}
where $\mu,\alpha > 0$ are tuning parameters. 
The uniqueness of the minimizer of \eqref{eq:1&2-norm TVmodel} is guaranteed
by the strict convexity of the objective function.

The unconstrained optimization problem \eqref{eq:1&2-norm TVmodel} can be reformulated as a constrained problem by
\begin{equation} \label{eq:reformulate2}
\begin{aligned}
&\argmin_{\bu,\d,\x,\y} \big({\lVert \x \rVert}_1+{\lVert \y \rVert}_1+\mu{\lVert \d \rVert}_1+ \alpha{\lVert \bu-\f \rVert}_2^2\big) \\ 
&\text{subject to }\quad \x=D_x \bu, ~ \y=D_y \bu, ~\text{and}~ \d=\f-\bu.
\end{aligned}
\end{equation}
The constrained problem \eqref{eq:reformulate2} can be further reformulated as an unconstrained problem by the quadratic penalty method as
\begin{equation} \label{eq:reformulate3}
\begin{aligned}
&\argmin_{\bu,\d,\x,\y} \Big({\lVert \x \rVert}_1+{\lVert \y \rVert}_1+\mu{\lVert \d \rVert}_1+ \alpha{\lVert \bu-\f \rVert}_2^2 \\
&+\lambda\big({\lVert \x-D_x \bu \rVert}_2^2+{\lVert \y-D_y \bu \rVert}_2^2+{\lVert \d-(\f-\bu) \rVert}_2^2\big)\Big),
\end{aligned}
\end{equation}
where $\lambda>0$ is a suitable coefficient. 
To solve the problem \eqref{eq:reformulate3}, we adopt the split Bregman iteration formulated as
\begin{subequations} \label{eq:splitBregman}

\begin{equation} \label{eq:split-Bregman iter2}
\begin{aligned} 
&\Big(\bu^{(k+1)},\d^{(k+1)},\x^{(k+1)},\y^{(k+1)}\Big)
=\argmin_{\bu,\d,\x,\y} \bigg({\lVert \x \rVert}_1+{\lVert \y \rVert}_1 \\
&+\mu{\lVert \d \rVert}_1+ \alpha{\lVert \f-\bu \rVert}_2^2 +\lambda\Big({\big\lVert \d-(\f-\bu)-\b_1^{(k)} \big\rVert}_2^2 \\
&+{\big\lVert \x-D_x \bu-\b_2^{(k)} \big\rVert}_2^2+{\big\lVert \y-D_y \bu-\b_3^{(k)} \big\rVert}_2^2\Big)\bigg),
\end{aligned}
\end{equation}
where $\b_1^{(k)}$, $\b_2^{(k)}$ and $\b_3^{(k)}$ are known as Bregman vectors that satisfy
\begin{align*}
\b_1^{(k+1)}&=\b_1^{(k)}-\big(\d^{(k+1)}-(\f-\bu^{(k+1)})\big),\\
\b_2^{(k+1)}&=\b_2^{(k)}-\big(\x^{(k+1)}-D_x \bu^{(k+1)}\big),\\
\b_3^{(k+1)}&=\b_3^{(k)}-\big(\y^{(k+1)}-D_y \bu^{(k+1)}\big).
\end{align*}
The problem \eqref{eq:split-Bregman iter2} can be split into 4 sub-problems
\begin{align}
\bu^{(k+1)} &= \argmin_{\bu} \Big( \alpha{\lVert \f-\bu \rVert}_2^2+\lambda\big( {\big\lVert \d^{(k)}-(\f-\bu)-\b_1^{(k)} \big\rVert}_2^2 \nonumber \\
&+{\big\lVert \x^{(k)}-D_x \bu-\b_2^{(k)} \big\rVert}_2^2+{\big\lVert \y^{(k)}-D_y \bu-\b_3^{(k)} \big\rVert}_2^2\big)\Big),\label{eq:subproblem-u2}\\
\d^{(k+1)} &= \argmin_{\d} \Big(\mu{\lVert \d \rVert}_1+\lambda{\big\lVert \d-(\f-\bu^{(k+1)})-\b_1^{(k)} \big\rVert}_2^2\Big),\label{eq:subproblem-d2}\\
\x^{(k+1)} &= \argmin_{\x} \Big({\lVert \x \rVert}_1
+\lambda{\big\lVert \x-D_x \bu^{(k+1)}-\b_2^{(k)} \big\rVert}_2^2\Big),\label{eq:subproblem-x2}\\
\y^{(k+1)} &= \argmin_{\y} \Big({\lVert \y \rVert}_1+\lambda{\big\lVert \y-D_y \bu^{(k+1)}-\b_3^{(k)} \big\rVert}_2^2\Big).\label{eq:subproblem-y2}
\end{align}
\end{subequations}
In particular, the problem \eqref{eq:subproblem-u2} is equivalent to solving the linear system
\begin{align*}
    &\big[\lambda(I_{mn}+D_x^\top D_x+D_y^\top D_y)+\alpha I_{mn}\big] \bu^{(k+1)} \\
    &=\lambda\big( \f-\d^{(k)}+\b_1^{(k)}+D_x^\top(\x^{(k)}-\b_2^{(k)})+D_y^\top(\y^{(k)}-\b_3^{(k)})\big) \\
    &+\alpha \f.
\end{align*}
The problems \eqref{eq:subproblem-d2}--\eqref{eq:subproblem-y2} are explicitly solved by the soft shrinkage function \cite{PaBo14}

\begin{align} \label{eq:shrink1}
\mathrm{shrink_1}(\v,\gamma)
&=\argmin_{\bu} \left({\gamma\lVert \bu \rVert}_1+\frac{1}{2}{\big\lVert \bu-\v \big\rVert}_2^2\right) \\
&=\Big[\mathrm{sign}(\v_i) \max\big\{0,|\v_i|-\gamma\big\}\Big]^{n}_{i=1}.
\end{align} 

The detailed computational procedure is summarized in Algorithm \ref{alg:2}.

\begin{algorithm}[H]
\caption{MixTV model for image denoising}
\label{alg:2}
\begin{algorithmic}[1] 
\Require A noisy image $F$, parameters $\lambda$, $\mu$, $\alpha$, a maximal number of iterations $M$, and a tolerance $\varepsilon$.
\Ensure A denoised image $U$.
\State Let $m$ and $n$ be the height and width of the image $F$. 
\State Let $\f = \mathrm{vec}(F)$.
\State Let $D_x = G_n\otimes I_m$ and $D_y = I_n\otimes G_m$ as in \eqref{eq:D}. 

\State Let $A = \lambda (I_{mn} + D_x^\top D_x + D_y^\top D_y) + \alpha I_{mn}$. 
\State Let $\b_1=\0$, $\b_2=\0$ and $\b_3=\0$.
\State Let $\mathbf{u}_0 = \mathbf{f}, \mathbf{d}=0, \mathbf{x} = 0$, and $\mathbf{y}=0$.
\State Let $i=0$ and $e = \infty$.
\While{ $e>\varepsilon$ and $i<M$ }
    \State Update $i \gets i + 1$.
    \State Update $ \r \gets \lambda \big(\f - \d + \b_1 + D_x^\top (\x-\b_2) + D_y^\top (\y - \b_3) \big) + \alpha \f $.
    \State Update $\bu$ by solving $A \bu = \r$.
    \State Update $\d \gets \mathrm{shrink}_1(\f - \bu + \b_1, \frac{\mu}{2\lambda}) $ as in \eqref{eq:shrink1}.
    \State Update $\x \gets \mathrm{shrink}_1(D_x \bu + \b_2, \frac{1}{2\lambda}) $.
    \State Update $\y \gets \mathrm{shrink}_1(D_y \bu + \b_3, \frac{1}{2\lambda}) $.
    \State Update $\b_1 \gets \b_1 + \f - \bu - \d $.
    \State Update $\b_2 \gets \b_2 + D_x \bu - \x $.
    \State Update $\b_3 \gets \b_3 + D_y \bu - \y $.
    \State Update $e \gets \|\bu - \bu_0\|_2$.
    \State Update $\bu_0 \gets \bu$.
\EndWhile
\State \Return $U = \mathrm{vec}^{-1}(\bu)$.
\end{algorithmic}
\end{algorithm}

\section{Convergence Analysis}
\label{sec:4}
In this section, we prove the convergence of the split Bregman iteration, Algorithm \ref{alg:2}, for the MixTV model \eqref{eq:1&2-norm TVmodel}, by applying conventional techniques in \cite{JiZZ09}.

In general, we consider the Bregman iteration of the form
\begin{subequations} \label{eq:Bregman}
\begin{equation} \label{eq:Breg-iter1}
\bu^{(k+1)}=\argmin_{\bu} \left(E(\bu)- {\g^{(k)}}^\top(\bu-\bu^{(k)}) +H(\bu)\right)
\end{equation}
and
\begin{equation} \label{eq:Breg-iter2}
    \g^{(k+1)}=\g^{(k)}-\nabla H(\bu^{(k+1)}),
\end{equation}
\end{subequations}
where $E$ and $H$ are continuous convex functions on $\mathbb{R}^{n}$. 
First, we introduce some useful lemmas for the proof of the convergence of the Bregman iteration \eqref{eq:Bregman}.
\begin{lemma}[Theorem 1~\cite{JiZZ09}] 
Given initial vectors $\g^{(0)}$ and $\bu^{(0)}$ in $\mathbb{R}^{n}$. Let $\bu^{(k+1)}$ and $\g^{(k+1)}$, for $k\in\mathbb{N}$, be vectors obtained from the Bregman iteration \eqref{eq:Breg-iter1} and \eqref{eq:Breg-iter2}. Then
\begin{align}
&\g^{(k)}\in\partial E(\bu^{(k)}) \\
&:= \left\{ \g\in\mathbb{R}^n \,\left|\, E(\bu)-E(\bu^{(k)})- \g^\top(\bu-\bu^{(k)}) \geq 0, \forall \bu\in\mathbb{R}^{n} \right.\right\}.
\end{align}
and $H(\bu^{(k+1)})\leq H(\bu^{(k)})$, for $k\in\mathbb{N}\backslash\{0\}$. Moreover, if $\min_{\bu\in\mathbb{R}^{n}} H(\bu)=0$, then
\begin{equation*}
\lim_{k\rightarrow\infty}H(\bu^{(k)})=0.
\end{equation*}
\label{lemma:1}
\end{lemma}

In particular, we consider the case
$H(\bu) = \rho \left\lVert A\bu-\b \right\rVert^2_2$,
where $\bu\in\mathbb{R}^{n}$, $\rho>0$, $A\in\mathbb{R}^{m\times n}$ and $\b\in\mathbb{R}^{m}$. In this case, the Bregman iteration is equivalent to the iteration scheme
\begin{subequations} \label{eq:Breg-iter34}
\begin{equation} \label{eq:Breg-iter3}
    \bu^{(k+1)}=\argmin_{\bu} \left(E(\bu)+ \rho \lVert A\bu-\b^{(k)} \rVert^2_2\right)
\end{equation}
and
\begin{equation} \label{eq:Breg-iter4}
    \b^{(k+1)}=\b^{(k)}+\b-A\bu^{(k+1)},
\end{equation}
provided an initial vector $\b^{(0)}\in\mathbb{R}^{m}$.
\end{subequations}
The following lemma provides sufficient conditions for the convergence of the iteration \eqref{eq:Breg-iter34} to the minimizer of the objective function $E$ under the constraint $A\bu = \b$.
\begin{lemma}[Theorem 2~\cite{JiZZ09}] \label{lemma:2}
Suppose the minimization problem
\begin{equation*}
    \min_{\bu} E(\bu)\quad \text{subject to }\quad A\bu=\b
\end{equation*}
has a unique solution $\tilde{\bu}$. For $k\in\mathbb{N}$, we let $\bu^{(k+1)}$ and $\b^{(k+1)}$ be given as in \eqref{eq:Breg-iter3} and \eqref{eq:Breg-iter4}. Then $\lim_{k\rightarrow\infty}\bu^{(k)}=\tilde{\bu}$, provided the following three conditions are satisfied:
\begin{enumerate}
\item[(i)] $\lim_{k\rightarrow\infty}(\bu^{(k+1)}-\bu^{(k)}) = \0$,
\item[(ii)] $\{\bu^{(k)}\}_{k\in\mathbb{N}}$ is a bounded sequence in $\mathbb{R}^{n}$,
\item[(iii)] $\{\b^{(k)}\}_{k\in\mathbb{N}}$ is a bounded sequence in $\mathbb{R}^{m}$.
\end{enumerate}
\end{lemma}

The following lemma is useful for verification of condition (i) in Lemma \ref{lemma:2}.
\begin{lemma}[Theorem 3~\cite{JiZZ09}] \label{lemma:3}
Let $E(\bu) = F(\bu) + \alpha \|\bu - \f\|_2^2$ with $\alpha > 0$ and $F$ being a continuous convex function on $\mathbb{R}^n$.

If $\{\bu^{(k+1)}\}_{k\in\mathbb{N}}$ and $\{\g^{(k+1)}\}_{k\in\mathbb{N}}$ are the sequences given by the Bregman iteration \eqref{eq:Breg-iter1} and \eqref{eq:Breg-iter2}, respectively, then
\begin{equation*}
\lim_{k\rightarrow\infty} \left\|\bu^{(k+1)}-\bu^{(k)} \right\|_2 = 0.
\end{equation*}
\end{lemma}

Now we are ready to prove the convergence of the associated Bregman iteration, Algorithm \ref{alg:2}, for computing the minimizer of the new model \eqref{eq:1&2-norm TVmodel}. 
Let $\f$ be the vectorization of an observed noisy image of size $m$-by-$n$. 
Based on the techniques in \cite{JiZZ09}, the minimizer of the objective function \eqref{eq:reformulate3} can be carried out by the Bregman iteration:

\begin{equation} \label{eq:udxy-iter}
\begin{aligned} 
    &\Big( \bu^{(k+1)}, \d^{(k+1)}, \x^{(k+1)}, \y^{(k+1)}\Big) \\
    &= \argmin_{\bu,\d,\x,\y} \Big( E(\bu,\d,\x,\y)+ H^{(k)}(\bu,\d,\x,\y)\Big),
\end{aligned}
\end{equation}
where $E(\bu,\d,\x,\y)= {\lVert \x \rVert}_1 + {\lVert \y \rVert}_1 + \mu{\lVert \d \rVert}_1 + \alpha{\lVert \bu-\f \rVert}_2^2$, and

\begin{equation*}
\begin{aligned}
&H^{(k)}(\bu,\d,\x,\y)=\lambda\Big({\big\lVert \d-(\f-\bu)-\b_1^{(k)} \big\rVert}_2^2 \\
&+{\big\lVert \x-D_x \bu-\b_2^{(k)} \big\rVert}_2^2+{\big\lVert \y-D_y \bu-\b_3^{(k)} \big\rVert}_2^2\Big),
\end{aligned}
\end{equation*}
in which
\begin{equation} \label{eq:b-iter}
\begin{aligned}
\b_1^{(k+1)} &= \b_1^{(k)}-\big(\d^{(k+1)}-(\f-\bu^{(k+1)})\big),\\
\b_2^{(k+1)} &= \b_2^{(k)}-\big(\x^{(k+1)}-D_x \bu^{(k+1)}\big),\\
\b_3^{(k+1)} &= \b_3^{(k)}-\big(\y^{(k+1)}-D_y \bu^{(k+1)}\big),
\end{aligned}
\end{equation}
with $\b_1^{(0)} = \b_2^{(0)} = \b_3^{(0)} = \0$.
In the following theorem, we apply the preceding lemmas to analyze the convergence of the Bregman iteration.
\begin{theorem} \label{thm:1}
Let $\{(\bu^{(k+1)}, \d^{(k+1)}, \x^{(k+1)}, \y^{(k+1)})\}_{k\in\mathbb{N}}$ be sequences given by the iteration scheme \eqref{eq:udxy-iter} and let $\tilde{\bu}$ be the unique solution to the minimization problem \eqref{eq:1&2-norm TVmodel}.
Then 
$$
\lim_{k\rightarrow\infty}\bu^{(k)}=\tilde{\bu}.
$$
\end{theorem}

Theorem \ref{thm:1} provides a solid foundation of the split-Bregman iteration, Algorithm \ref{alg:2}, for computing the solution to the model \eqref{eq:1&2-norm TVmodel}, which represents a denoised image. In the next section, we will demonstrate the effectiveness of Algorithm \ref{alg:2} compared to other TV-based algorithms. We provide the detailed proof in the Appendix.

\section{Experiments}
\label{sec:5}

Now we present the experiments of Algorithm \ref{alg:2} and comparisons with other methods. \\

\noindent \textbf{Implementation Detail.}
We evaluate all methods on images with a resolution of $250 \times 250$ under four types of synthetic noise: Gaussian noise, salt and pepper (S\&P) noise, Poisson noise, and speckle noise. Following the default setting in MATLAB's \texttt{imnoise} function, Gaussian noise is generated with zero mean and variance $0.01$, S\&P noise is applied with noise density $0.05$, Poisson noise is generated according to the image intensity values, and speckle noise is generated as multiplicative noise with variance $0.05$. For all compared methods, we set the parameters to $\lambda = 1$, $\mu = 1$, $\alpha = 1$. For RGB images, we apply the proposed model independently to each color channel. \\

\noindent \textbf{Metrics.}
Let $U$ and $T$ be the denoised and the original images of size $m\times n$, respectively. The similarity between two images is commonly measured by the following measurements: peak signal-to-noise ratio (PSNR), structural similarity (SSIM). It is known that PSNR and SSIM are commonly regarded as important indicators of the similarity of images. In our numerical results, we use the product of PSNR and SSIM as our similarity measurement, defined as $\mathrm{PPS}(U,T)=\mathrm{PSNR}(U,T) \cdot \mathrm{SSIM}(U,T)$. \\

\noindent \textbf{Benchmark.}
We demonstrate numerical experiments of state-of-the-art algorithms for image denoising for various noise types. The benchmark images used in the experiments are shown in Figure \ref{fig:images}. The \textbf{dog} image typically symbolizes the ubiquitous realities captured in real-world scenarios. The \textbf{painting} image represents common occurrences of digital art and illustrations. The \textbf{mosaic} image specifically serves to challenge TV models that are inept in processing multiple boundaries. The \textbf{cameraman} image is a commonly utilized reference in the evaluation of denoising methodologies. \\

\begin{figure}[t]
    \centering
    \footnotesize 
    \stackunder[5pt]{\includegraphics[width=0.23\linewidth]{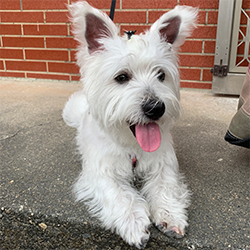}}{dog}
    \hskip 0.01\linewidth
    \stackunder[5pt]{\includegraphics[width=0.23\linewidth]{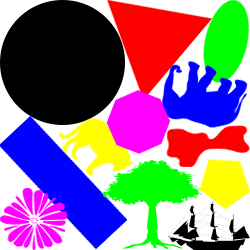}}{painting}
    \hskip 0.01\linewidth
    \stackunder[5pt]{\includegraphics[width=0.23\linewidth]{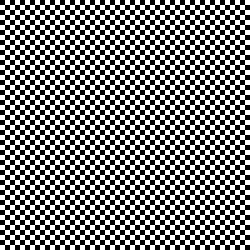}}{mosaic}
    \hskip 0.01\linewidth
    \stackunder[5pt]{\includegraphics[width=0.23\linewidth]{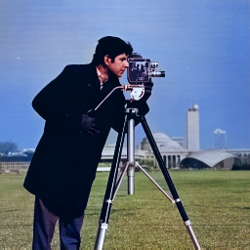}}{cameraman}
    \caption{Benchmark images used in our experiments.}
    \label{fig:images}
    \vspace*{-15pt}
\end{figure}

\noindent \textbf{Results.} In the experiments, we consider several types of noise, including Gaussian noise, S\&P noise, Poisson noise, speckle noise, uniform noise, and their combinations. Qualitative results under these noise settings are shown in Figure~\ref{fig:qualitative_result}. All the experiments are performed in MATLAB on a personal computer. In Table \ref{tab:1}, we demonstrate the similarity of the denoised images and the ground truth in the measurement of PPS averaged over the four images. The rationale for overlaying the 1-norm model with the 2-norm Isotropic and Anisotropic models in the denoising process lies in our intent to ascertain whether such combinations can achieve comparable performance to the mixed-norm model. Notably, our proposed MixTV model exhibits superior performance compared to other models.

Furthermore, when excluding DnCNN from consideration, it could be asserted that the MixTV model decidedly outperforms other TV-based models in effectiveness and efficiency. It is evident that the MixTV model exhibits robustness across various noise types.

\section{Conclusion}
\label{sec:6}
The TV model is a well-known image-denoising model. Its success is based on strong theoretical properties and efficient numerical schemes. Conventionally, the regularity term has been the main focus of research. In this paper, we also incorporate the data fidelity term into consideration. We introduce a new TV model by combining the 1-norm and 2-norm data fidelity terms. The experimental results indicate that the proposed new model often yields better outcomes than other commonly used TV models. Numerical experiments demonstrate the effectiveness and efficiency of the new model across various kinds of images and different types of noise and their combinations. Furthermore, we prove the convergence of the new model. The remarkable outcomes obtained in this study strengthen the practicality and effectiveness of the TV method in image processing.

\begin{table*}[t]
\resizebox{\textwidth}{!}{
\begin{tabular}{lrrrrrrrrr>{\columncolor{blue!8}}r}
\hline
\multirow{2}{*}{\textbf{Noise Type}} & \multirow{2}{*}{\textbf{Noisy}} & \multirow{2}{*}{\footnotesize \textbf{1-norm}} & \multirow{2}{*}{\footnotesize \textbf{Isotropic}} & \multirow{2}{*}{\footnotesize \textbf{Anisotropic}} & {\footnotesize \textbf{1-norm}} & {\footnotesize \textbf{1-norm}} & {\footnotesize \textbf{Isotropic}} & {\footnotesize \textbf{Anisotropic}} & \multirow{2}{*}{\textbf{DnCNN}} & {\footnotesize \textbf{MixTV}} \\
  & & & & & {\footnotesize \textbf{+ Isotropic}} & {\footnotesize \textbf{+ Anisotropic}} & {\footnotesize \textbf{+ 1-norm}} & {\footnotesize \textbf{+ 1-norm}} & & \textbf{(Ours)} \\
\hline

Gaussian            &	16.16 & 19.41 & 9.00 & 8.53 & 9.21 & 8.74 & 8.99 & 8.53 & \textbf{21.87} & 21.15  \\
S\&P                &	11.58 & 20.48 & 9.03 & 8.54 & 6.52 & 9.12 & 9.02 & 8.54 & 15.86 & \textbf{22.84}  \\
Poisson             &	26.33 & 22.51 & 9.60 & 9.04 & 9.54 & 9.03 & 9.59 & 9.04 & \textbf{28.21} & 24.86  \\
speckle             &	14.44 & 17.52 & 9.09 & 8.56 & 9.08 & 8.62 & 9.08 & 8.55 & \textbf{21.07} & 19.42  \\
uniform             &	11.00 & 18.66 & 8.49 & 8.04 & 6.46 & 9.03 & 8.48 & 8.04 & 16.81 & \textbf{21.67}  \\
Gaussian + S\&P     &	9.53 & 18.04 & 8.51 & 8.06 & 9.18 & 8.71 & 8.50 & 8.06 & 16.36 & \textbf{19.29}  \\
Gaussian + Poisson  &	15.10 & 18.88 & 8.92 & 8.45 & 9.11 & 8.64 & 8.91 & 8.44 & \textbf{21.15} & 20.51  \\
Gaussian + speckle  &	10.80 & 15.96 & 8.47 & 8.03 & 8.66 & 8.24 & 8.46 & 8.03 & \textbf{18.11} & 17.39  \\
Gaussian + uniform  &	9.03 & 15.79 & 7.99 & 7.59 & 8.78 & 8.37 & 7.99 & 7.59 & 15.92 & \textbf{17.61}  \\
S\&P + Gaussian     &	9.87 & 17.93 & 8.49 & 8.05 & 9.12 & 8.67 & 8.48 & 8.05 & 16.60 & \textbf{19.17}  \\
S\&P + Poisson      &	11.23 & 19.82 & 8.92 & 8.45 & 9.51 & 8.99 & 8.92 & 8.45 & 16.62 & \textbf{21.94}  \\
S\&P + speckle      &	9.11 & 16.64 & 8.46 & 8.02 & 9.00 & 8.54 & 8.45 & 8.02 & 15.91 & \textbf{18.18}  \\
S\&P + uniform      &	7.73 & 17.48 & 7.99 & 7.59 & 6.46 & 9.00 & 7.99 & 7.59 & 14.51 & \textbf{20.12}  \\
Poisson + Gaussian  &	15.04 & 18.80 & 8.93 & 8.46 & 9.12 & 8.65 & 8.92 & 8.46 & \textbf{21.16} & 20.41  \\
Poisson + S\&P   &	11.06 & 19.71 & 8.94 & 8.46 & 9.52 & 9.00 & 8.93 & 8.46 & 16.57 & \textbf{21.99}  \\
Poisson + speckle   &	13.71 & 17.08 & 9.01 & 8.50 & 8.98 & 8.53 & 9.00 & 8.49 & \textbf{20.49} & 18.88  \\
Poisson + uniform   &	10.49 & 17.94 & 8.40 & 7.96 & 9.36 & 8.86 & 8.40 & 7.96 & 16.81 & \textbf{20.71}  \\
speckle + Gaussian  &	10.88 & 16.04 & 8.44 & 8.00 & 8.65 & 8.23 & 8.43 & 8.00 & \textbf{18.01} & 17.44  \\
speckle + S\&P      &	8.82 & 16.71 & 8.48 & 8.04 & 9.04 & 8.59 & 8.47 & 8.04 & 15.60 & \textbf{18.23}  \\
speckle + Poisson   &	13.90 & 17.15 & 8.96 & 8.44 & 8.92 & 8.47 & 8.96 & 8.44 & \textbf{20.45} & 18.94  \\
speckle + uniform   &	8.37 & 15.40 & 7.95 & 7.55 & 8.71 & 8.28 & 7.95 & 7.55 & 15.36 & \textbf{17.09}  \\
uniform + Gaussian  &	9.01 & 15.88 & 8.01 & 7.60 & 8.81 & 8.40 & 8.00 & 7.60 & 15.90 & \textbf{17.65}  \\
uniform + S\&P      &	7.64 & 17.55 & 8.01 & 7.60 & 6.47 & 9.01 & 8.00 & 7.60 & 14.49 & \textbf{20.20}  \\
uniform + Poisson   &	10.49 & 18.05 & 8.39 & 7.95 & 9.36 & 8.86 & 8.38 & 7.95 & 16.79 & \textbf{20.72}  \\
uniform + speckle   &	8.33 & 15.40 & 7.96 & 7.57 & 8.71 & 8.29 & 7.95 & 7.57 & 15.30 & \textbf{17.09}  \\
\hline
\end{tabular}}
\caption{Comparison of various variational methods for image denoising in terms of the PPS score.}
\label{tab:1}
\vspace*{-15pt}
\end{table*}

\begin{figure*}[h]
    \centering
    \includegraphics[width=0.85\linewidth]{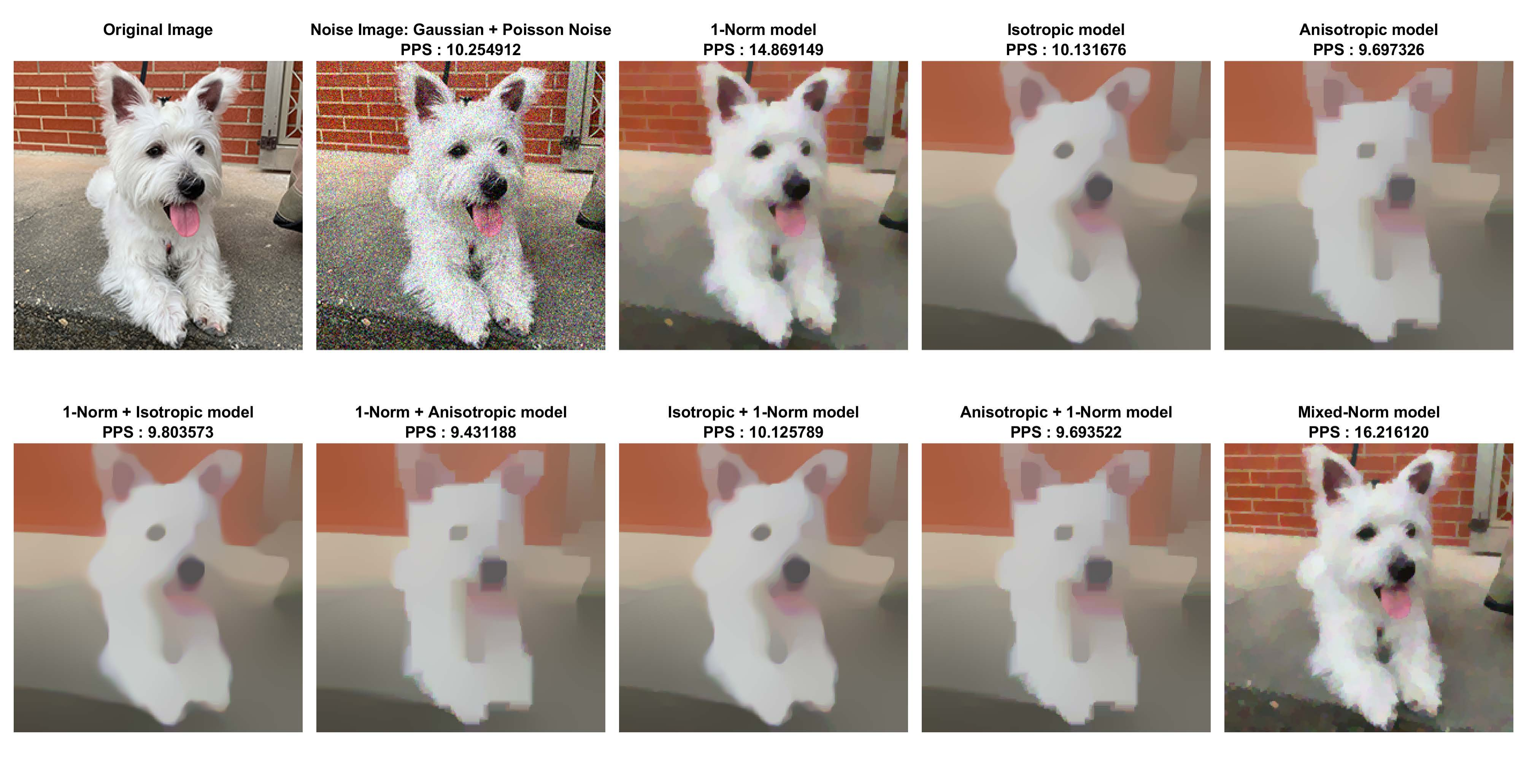}
    \includegraphics[width=0.85\linewidth]{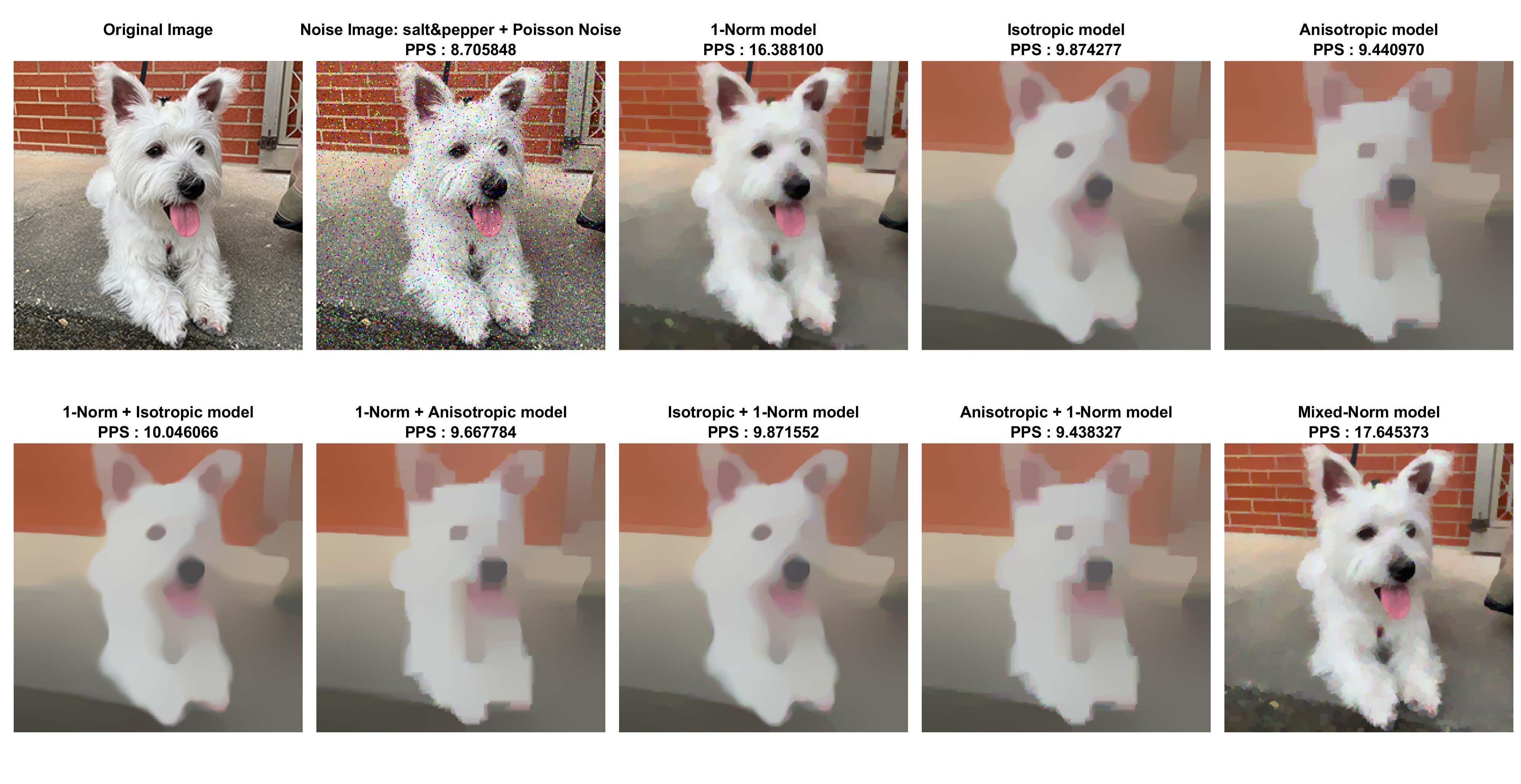}
    \includegraphics[width=0.85\linewidth]
    {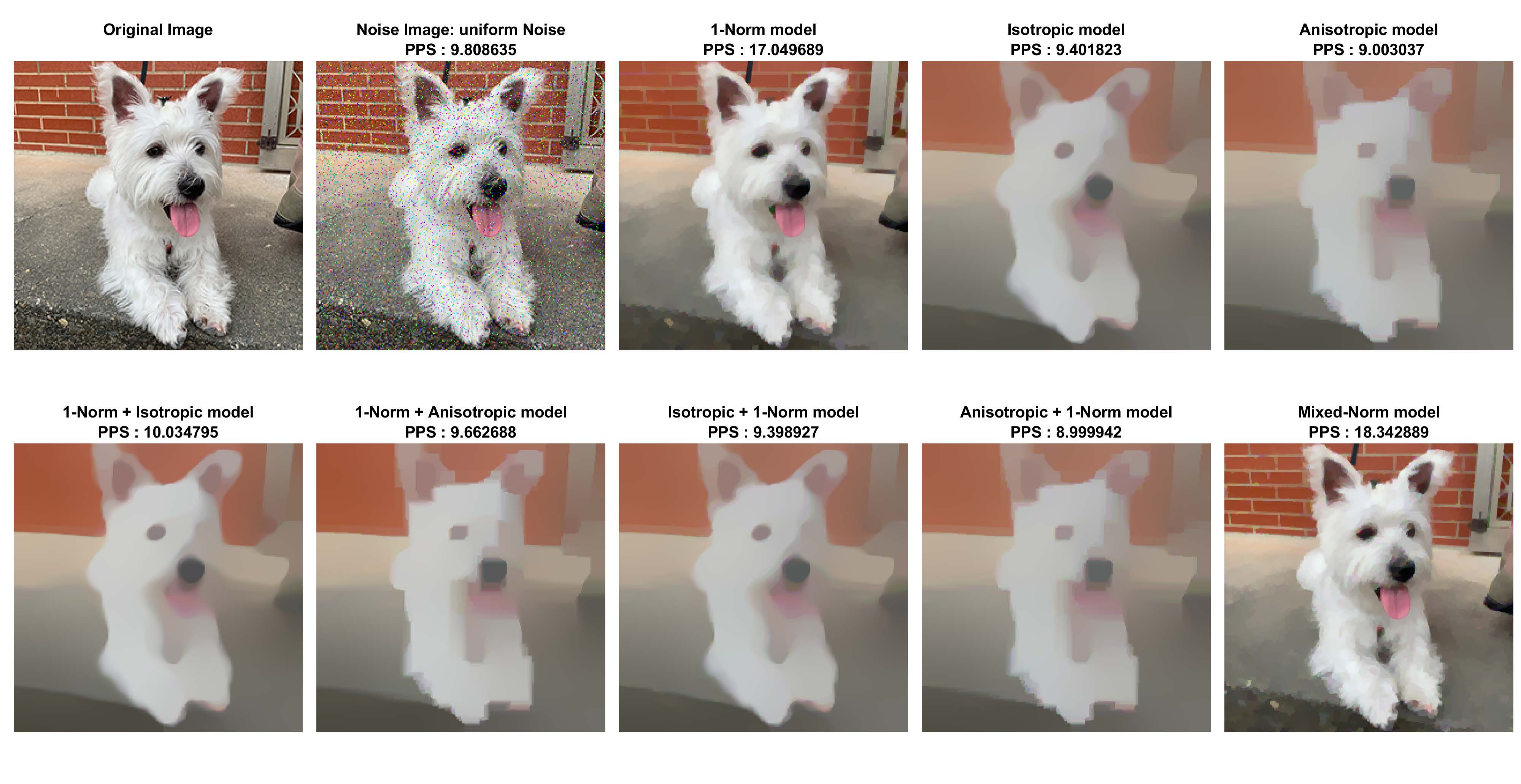}
    \caption{Qualitative results of various variational methods for image denoising.}
    \label{fig:qualitative_result}
\end{figure*}

\bibliographystyle{plain}
\bibliography{references}

\clearpage
\onecolumn
\setcounter{section}{0}

\begin{center}
    \LARGE{\textbf{Appendix}}
\end{center}

\

\section{Notations}
In this paper, we use the following notations. Other notations will be clearly defined when they appear.
\begin{itemize}
    \item $\mathbb{R}$ and $\mathbb{N}$ denote the sets of real numbers and natural numbers, respectively.
    \item Bold letters, e.g., $\b$, $\f$, $\bu$, denote vectors.
    \item Capital letters, e.g., $A$, $F$, $U$, denote matrices.
    \item $\f_{i}$ denotes the $i$th entry of the vector $\f$.
    \item $A_{i,j}$ denotes the $(i,j)$th entry of the matrix $A$.
    \item $I_n$ denotes the identity matrix of size $n$-by-$n$.
    \item $\0$ denotes the zero vector of an appropriate length.
    \item $\text{vec}(A)=(A_{1,1}, \ldots, A_{m,1},
    A_{1,2}, \ldots, A_{m,2},
    \ldots, A_{1,n}, \ldots, A_{m,n})^\top$ denotes the vectorization of $A\in\mathbb{R}^{m\times n}$.
\end{itemize}

\section{Mathematical Details}
Following \cite{JiZZ09}, we provide the proofs of the propositions, lemmas, and theorems stated in the main paper below.

\begin{proposition} \label{prop:uniqueness}
The minimization problem (\textcolor{red}{5}) has a unique solution.
\end{proposition}
\begin{proof}
The composition of a convex function with an affine map is convex, and the sum of a convex function and a strictly convex function is strictly convex. Hence, (\textcolor{red}{5}) is strictly convex and admits a unique minimizer. 
\end{proof}

\begin{theorem} \label{thm:1}
Let $\{(\bu^{(k+1)}, \d^{(k+1)}, \x^{(k+1)}, \y^{(k+1)})\}_{k\in\mathbb{N}}$ be sequences given by the iteration scheme (\textcolor{red}{15}) and let $\tilde{\bu}$ be the unique solution to the minimization problem (\textcolor{red}{5}).
Then 
$$
\lim_{k\rightarrow\infty}\bu^{(k)}=\tilde{\bu}.
$$
\end{theorem}
\begin{proof}
By Proposition \ref{prop:uniqueness}, there exists a unique solution $\tilde{\bu}$ to the minimization problem (\textcolor{red}{5}). From the equivalences of the reformulations (\textcolor{red}{6}) and (\textcolor{red}{7}), we conclude that $(\tilde{\bu}, \tilde{\d}, \tilde{\x}, \tilde{\y})$ is the unique solution to the minimization problem (\textcolor{red}{7}) that satisfies $\tilde{\d} = \f-\tilde{\bu}$, $\tilde{\x} = D_x\tilde{\bu}$ and $\tilde{\y} = D_y\tilde{\bu}$.
By applying Lemma \textcolor{red}{1} to (\textcolor{red}{7}) with 
$$
E(\bu,\d,\x,\y)= {\lVert \x \rVert}_1 + {\lVert \y \rVert}_1 + \mu{\lVert \d \rVert}_1 + \alpha{\lVert \bu-\f \rVert}_2^2,
$$
and
$$
H(\bu,\d,\x,\y)=\lambda\Big({\big\lVert \d-(\f-\bu) \big\rVert}_2^2+{\big\lVert \x-D_x \bu \big\rVert}_2^2+{\big\lVert \y-D_y \bu \big\rVert}_2^2\Big),
$$
we obtain
\begin{equation} \label{eq:Thm4-1}
\lim_{k\rightarrow\infty}{\big\lVert \d^{(k)}-(\f-\bu^{(k)}) \big\rVert}_2^2=0, ~
\lim_{k\rightarrow\infty}{\big\lVert \x^{(k)}-D_x \bu^{(k)} \big\rVert}_2^2=0, ~ \text{and} ~
\lim_{k\rightarrow\infty}{\big\lVert \y^{(k)}-D_y \bu^{(k)} \big\rVert}_2^2=0.
\end{equation}
By applying Lemma \textcolor{red}{3} to $E(\mathbf{u}, \mathbf{d}, \mathbf{x}, \mathbf{y})$, we obtain
\begin{equation*}
\lim_{k\rightarrow\infty}\left\|\bu^{(k+1)}-\bu^{(k)}\right\| = 0.
\end{equation*}
By multiplying the sequence by $D_x$ and $D_y$, respectively, it follows that
\begin{equation} \label{eq:Proof1}
\lim_{k\rightarrow\infty}\left\|D_x \bu^{(k+1)}-D_x \bu^{(k)}\right\| = 0
~\text{ and }~ 
\lim_{k\rightarrow\infty}\left\|D_y \bu^{(k+1)}-D_y \bu^{(k)}\right\| = 0.
\end{equation}
From \eqref{eq:Thm4-1}, \eqref{eq:Proof1} and the triangle inequality, we obtain
\begin{equation*}
\lim_{k\rightarrow\infty}\left\|\d^{(k+1)}-\d^{(k)}\right\|=0, ~
\lim_{k\rightarrow\infty}\left\|\x^{(k+1)}-\x^{(k)}\right\| = 0, ~\text{and}~
\lim_{k\rightarrow\infty}\left\|\y^{(k+1)}-\y^{(k)}\right\| = 0.
\end{equation*}
To prove the convergence $\lim_{k\rightarrow\infty}\bu^{(k)}=\tilde{\bu}$, by applying Lemma \textcolor{red}{2}, it suffices to show that the sequences $\{(\bu^{(k)}, \d^{(k)}, \x^{(k)}, \y^{(k)})\}_{k\in\mathbb{N}\backslash\{0\}}$, $\{\b^{(k)}_1\}_{k\in\mathbb{N}\backslash\{0\}}$, $\{\b^{(k)}_2\}_{k\in\mathbb{N}\backslash\{0\}}$, and $\{\b^{(k)}_3\}_{k\in\mathbb{N}\backslash\{0\}}$ are bounded.
From the optimality condition of (\textcolor{red}{15}), we have
\begin{align*}
\mu {\big\lVert \d^{(k+1)} \big\rVert}_1+\lambda{\big\lVert \d^{(k+1)}- (\f - \bu^{(k+1)})-\b^{(k)}_1 \big\rVert}^2_2 & \leq \mu {\big\lVert \d \big\rVert}_1+\lambda{\big\lVert \d - (\f-\bu^{(k+1)})-\b^{(k)}_1 \big\rVert}^2_2 ~\text{for all $\d\in\mathbb{R}^{m n}$}, \\
{\big\lVert \x^{(k+1)} \big\rVert}_1+\lambda{\big\lVert \x^{(k+1)}-D_x \bu^{(k+1)}-\b^{(k)}_2 \big\rVert}^2_2 & \leq {\big\lVert \x\big\rVert}_1+\lambda{\big\lVert \x-D_x \bu^{(k+1)}-\b^{(k)}_2 \big\rVert}^2_2 ~\text{for all $\x\in\mathbb{R}^{m n}$}, \\
{\big\lVert \y^{(k+1)} \big\rVert}_1+\lambda{\big\lVert \y^{(k+1)}-D_y \bu^{(k+1)}-\b^{(k)}_3 \big\rVert}^2_2 & \leq {\big\lVert \y\big\rVert}_1+\lambda{\big\lVert \y-D_y \bu^{(k+1)}-\b^{(k)}_3 \big\rVert}^2_2 ~\text{for all $\y\in\mathbb{R}^{m n}$}.
\end{align*}
Thus, from the property of the soft shrinkage function defined in (\textcolor{red}{9}), we have
\begin{align*}
\d^{(k+1)} &= \mathrm{shrink}_1\left(\f-\bu^{(k+1)}+\b^{(k)}_1,\frac{\mu}{2\lambda}\right),\\
\x^{(k+1)} &= \mathrm{shrink}_1\left(D_x \bu^{(k+1)}+\b^{(k)}_2,\frac{1}{2\lambda}\right),\\
\y^{(k+1)} &= \mathrm{shrink}_1\left(D_y \bu^{(k+1)}+\b^{(k)}_3,\frac{1}{2\lambda}\right).
\end{align*}
These three equations together with (\textcolor{red}{16}) give
\begin{align*}
\b^{(k+1)}_1 &= \mathrm{cut}\left( \f-\bu^{(k+1)}+\b^{(k)}_1,\frac{\mu}{2\lambda}\right),\\
\b^{(k+1)}_2 &= \mathrm{cut}\left( D_x \bu^{(k+1)}+\b^{(k)}_2,\frac{1}{2\lambda}\right),\\
\b^{(k+1)}_3 &= \mathrm{cut}\left( D_y \bu^{(k+1)}+\b^{(k)}_3,\frac{1}{2\lambda}\right),
\end{align*}
where the cut function is defined as 
$$
\mathrm{cut}(\v,\gamma) = \v - \mathrm{shrink}_1(\v,\gamma) = \Big[\mathrm{sign}(\v_i) \min\big\{|\v_i|, \gamma\big\}\Big]_{i=1}^{mn}.
$$

It is clear that $\| \, \mathrm{cut}(\v,\gamma) \, \|_\infty \leq \gamma$. 
In particular,
\begin{align*}
{\big\lVert \b^{(k+1)}_1 \big\rVert}_\infty \leq \frac{\mu}{2\lambda}, ~ 
{\big\lVert \b^{(k+1)}_2 \big\rVert}_\infty \leq \frac{1}{2\lambda} ~ \text{ and } ~
{\big\lVert \b^{(k+1)}_3 \big\rVert}_\infty \leq \frac{1}{2\lambda}.
\end{align*}
Furthermore, the optimality condition of (\textcolor{red}{15}) implies 
\begin{equation*}
E(\bu^{(k+1)},\d^{(k+1)},\x^{(k+1)},\y^{(k+1)})+ H^{(k)}(\bu^{(k+1)},\d^{(k+1)},\x^{(k+1)},\y^{(k+1)}) \leq E(\bu,\d,\x,\y)+ H^{(k)}(\bu,\d,\x,\y)
\end{equation*}
holds true for all $(\bu,\d,\x,\y)$. In particular, we consider the case $\bu=\d=\x=\y=\0$. From the nonnegativity of $H^{(k)}$, we obtain
\begin{align*}
& {\big\lVert \x^{(k+1)} \big\rVert}_1 + {\big\lVert \y^{(k+1)} \big\rVert}_1 + \mu{\big\lVert \d^{(k+1)} \big\rVert}_1 + \alpha{\big\lVert \bu^{(k+1)}-\f \big\rVert}_2^2 \\
&= E(\bu^{(k+1)},\d^{(k+1)},\x^{(k+1)},\y^{(k+1)}) \\
&\leq E(\0,\0,\0,\0) + H^{(k)}(\0,\0,\0,\0) - H^{(k)}(\bu^{(k+1)},\d^{(k+1)},\x^{(k+1)},\y^{(k+1)}) \\
&\leq E(\0,\0,\0,\0) + H^{(k)}(\0,\0,\0,\0) \\
&= \alpha{\big\lVert \f \big\rVert}_2^2+\lambda\Big({\big\lVert \b_1^{(k)}+\f \big\rVert}_2^2+{\big\lVert  \b_2^{(k)} \big\rVert}_2^2+{\big\lVert \b_3^{(k)} \big\rVert}_2^2\Big) \\
& \leq (\alpha+2\lambda){\big\lVert \f \big\rVert}_2^2+\lambda\Big({2 \big \lVert \b_1^{(k)} \big\rVert}_2^2+{\big\lVert \b_2^{(k)} \big\rVert}_2^2+{\big\lVert \b_3^{(k)} \big\rVert}_2^2\Big) \\
& \leq (\alpha+2\lambda){\big\lVert \f \big\rVert}_2^2 + \frac{mn (\mu^2+1)}{2 \lambda}
\end{align*}
which is a finite value independent of $k$.

Therefore, $\{(\bu^{(k)}, \d^{(k)}, \x^{(k)}, \y^{(k)})\}_{k\in\mathbb{N}\backslash\{0\}}$ is a bounded sequence. 
    By Lemma \textcolor{red}{2}, we conclude that the sequence $\{(\bu^{(k)}, \d^{(k)}, \x^{(k)}, \y^{(k)})\}_{k\in\mathbb{N}}$ converges and thereby $\lim_{k\rightarrow\infty}\bu^{(k)}=\tilde{\bu}$.
\end{proof}

Theorem \ref{thm:1} provides a solid foundation of the split-Bregman iteration, Algorithm \textcolor{red}{1}, for computing the solution to the model (\textcolor{red}{5}), which represents a denoised image.

\section*{Acknowledgements}
This research is supported in part by the National Center for Theoretical Sciences and the National Science and Technology Council, Taiwan, under the grant number NSTC-113-2115-M-003-012-MY2.

\end{document}